\documentclass[
hf
]{ceurart}

\usepackage{listings}
\usepackage{enumitem}
\usepackage{algorithm}
\usepackage{algpseudocode}

\usepackage{tikz}
\usetikzlibrary{decorations.pathmorphing,shapes,positioning,arrows.meta,arrows,automata,positioning,calc,shadows}
\usepackage{caption}
\usepackage{subcaption}
\usepackage{amsthm}

\usepackage{xargs} 
\let\todo\relax
\usepackage[colorinlistoftodos]{todonotes}
\setuptodonotes{size=tiny}

\newcommandx{\complete}[2][1=]{\todo[linecolor=orange,backgroundcolor=orange!25,bordercolor=orange,#1]{Complete: #2}}
\newcommandx{\tocite}[2][1=]{\todo[linecolor=pink,backgroundcolor=pink!25,bordercolor=pink,#1]{Cite: #2}}
\newcommandx{\unsure}[2][1=]{\todo[linecolor=blue,backgroundcolor=blue!5,bordercolor=blue,#1]{Unsure: #2}}
\newcommandx{\change}[2][1=]{\todo[linecolor=red,backgroundcolor=red!25,bordercolor=red,#1]{Change: #2}}
\newcommandx{\info}[2][1=]{\todo[linecolor=olive,backgroundcolor=olive!25,bordercolor=olive,#1]{Info: #2}}
\newcommandx{\improvement}[2][1=]{\todo[linecolor=purple,backgroundcolor=purple!25,bordercolor=purple,#1]{Improve: #2}}
\newcommandx{\cut}[2][1=]{\todo[linecolor=yellow,backgroundcolor=yellow!25,bordercolor=yellow,#1]{Potential Cut: #2}}

\newcommand{\attacks}[0]{\rightarrow}
\newcommand{\supports}[0]{\Rightarrow}
\newcommand{\args}[0]{\mathcal{A}}

\usepackage{bm}

\newcommand{\x}[1]{x_{#1}}
\newcommand{\y}[1]{y_{#1}}

\newcommand{\argalpha}{a}
\newcommand{\argbeta}{b}
\newcommand{\arggamma}{c}

\newcommand{\xalpha}{\x{\argalpha}}
\newcommand{\xbeta}{\x{\argbeta}}
\newcommand{\xgamma}{\x{\arggamma}}

\newcommand{\yalpha}{\y{\argalpha}}

\newcommand{\ygamma}{\y{\arggamma}}

\newcommand{\case}[2]{(\x{#1}, \y{#2})}
\newcommand{\casealpha}{\case{\argalpha}{\argalpha}}
\newcommand{\casebeta}{\case{\argbeta}{\argbeta}}
\newcommand{\casegamma}{\case{\arggamma}{\arggamma}}

\newcommand{\casedefault}{(\x{\delta}, \delta)}

\newcommand{\argnew}{N}
\newcommand{\casenew}{\case{\argnew}{?}}
\newcommand{\xnew}{\x{\argnew}}

\newcommand{\af}[2][\empty]{$AF_{#1}(D, {#2})$}
\newcommand{\baf}[2][\empty]{$BAF_{#1}(D, {#2})$}

\newtheorem{theorem}{Theorem}[section]
\newtheorem{definition}{Definition}[section]

\begin{document}

\copyrightyear{2025}
\copyrightclause{Copyright for this paper by its authors.
  Use permitted under Creative Commons License Attribution 4.0
  International (CC BY 4.0).}

\conference{IARML@IJCAI'2025: Workshop on the Interactions between Analogical Reasoning and Machine Learning, at IJCAI'2025,  August, 2025, Montreal, Canada}

\title{Supported Abstract Argumentation for Case-Based Reasoning}

\author[1]{Adam Gould}[%
orcid=0009-0008-0972-7501,
email=adam.gould19@imperial.ac.uk
]
\cormark[1]
\author[1]{Gabriel de Olim Gaul}[%
email=gabriel.de-olim-gaul21@imperial.ac.uk
]
\author[1]{Francesca Toni}[%
orcid=0000-0001-8194-1459,
email=f.toni@imperial.ac.uk
]
\address[1]{Department of Computing, Imperial College London, UK}

\cortext[1]{Corresponding author.}

\begin{abstract}
We introduce \textit{Supported Abstract Argumentation for Case-Based Reasoning (sAA-CBR)}, a binary classification model in which 
past \textit{cases} engage in debates by \textit{arguing} in favour of their labelling and \textit{attacking} or \textit{supporting} those with opposing or agreeing labels. With supports, sAA-CBR overcomes the limitation of its precursor AA-CBR, which can contain extraneous 
cases (or \textit{spikes}) that are not included in the debates. We prove that sAA-CBR contains no spikes, without trading off key model properties.
\end{abstract}

\begin{keywords}
  Computational Argumentation \sep
  Case-Based Reasoning \sep
  Machine Learning \sep
\end{keywords}

\maketitle

% \section{Introduction}

% \todoin{

% \begin{itemize}
%     \item Do not be too technical in text - takes up too much space. Explain at a high level.
%     \item Do not explain AA-CBR - just go straight into sAA-CBR
%     \item Refer to its benefits over AA-CBR in text, rather than describing all of AA-CBR and its limitations first
% \end{itemize}

% }

% We introduce \textit{Supported Abstract Argumentation for Case-Based Reasoning (sAA-CBR)}, a binary classification model in which previously observed data points (or \textit{cases}) \textit{argue} in favour of their labelling, \textit{attacking} those with opposing labels and \textit{supporting} those with agreeing labels. By introducing supports, sAA-CBR overcomes the limitation of traditional AA-CBR, which can contain extraneous data points, known as \textit{spikes}, which are not included in the debate. We show that sAA-CBR contains no spikes, without trading off key model properties.

\textit{Abstract Argumentation for Case-Based Reasoning (AA-CBR)}~\cite{aa-cbr} has proved to be an effective and interpretable, binary classification model~\cite{DEAr,ANNA,preference-based-aacbr}. In AA-CBR, each %data point 
past case \textit{argues} in favour of its labelling and \textit{attacks} those with opposing labels. However, AA-CBR constrains attacks to those between cases of minimal difference, which can result in \textit{spikes}~\cite{monotonicity-and-noise-tolerance}, cases with no contributing factor to the classification of %the data point
new cases. %As a result
To address this, we introduce \textit{Supported %Abstract Argumentation for Case-Based Reasoning 
AA-CBR (sAA-CBR)}, which adds the \textit{supports} relation between cases with agreeing labels.

In sAA-CBR (as in AA-CBR)
labelled (past) cases form a \textit{casebase}, which can be used to debate the outcome for an unlabelled (new) case. The unlabelled case attacks any cases considered \textit{irrelevant} to it. We can then apply argumentation semantics \cite{DUNG-aa,bipolarity-in-argumentation} to determine which arguments are accepted or rejected. The
%model includes 
debate starts with a 
 \textit{default argument}, which argues for a default expected outcome of the new case. If the default argument is accepted, then the model classifies the new case with this outcome, otherwise it assigns the opposing outcome~\cite{aa-cbr}.

To determine the direction of attacks and supports, %the model
sAA-CBR must have a partial order defined over the casebase, determining a notion of \textit{exceptionality}: more exceptional cases attack or support less exceptional cases. A \textit{minimality} constraint on attacks and supports ensures that cases only relate to those most similar and that there are no superfluous relationships\footnote{The full definition of sAA-CBR can be found in the supplementary material.}.

Consider a simple assessment of a patient's diet. Let the features $A, B, C, D$ represent eating five fruits a day, exceeding daily recommended calories, drinking $\geq 8$ cups of water daily and frequent high-fat food intake. It is unclear if a new patient $C_{N}$, who presents with all of these features, has a healthy $(+)$ or unhealthy $(-)$ diet. By default, patients are expected to have an unhealthy diet; this default is represented by a case with no features: $(\emptyset, -)$. Using known outcomes from four previously observed patients, we can argue about the outcome for the unlabelled patient. Figure~\ref{fig:example} showcases the AA-CBR (Figure~\ref{fig:example-1}) and sAA-CBR (Figure~\ref{fig:example-2}) models. The superset partial order $\supseteq$ is used for the notion of exceptionalism. For instance, case $C_{3}$ contains a superset of $C_{1}$'s features and an opposing outcome, so $C_{3}$ attacks~$C_{1}$. Note that $C_{2}$ does not attack $C_{0}$, despite having a superset of features. This is because $C_{1}$ is more similar to $C_{0}$. This is what we mean when we enforce the minimality constraint
\footnote{For brevity, neither figure shows an example of irrelevance. A case, $C_{5} = (\{E\}, +)$, is irrelevant to $C_{N}$ as $E$ is absent in $C_{N}$. $C_{N}$ would attack $C_{5}$. 
Semantically, such an attack equates to removing $C_{5}$ from the framework.
}.

\newcommand{\examplew}{1}

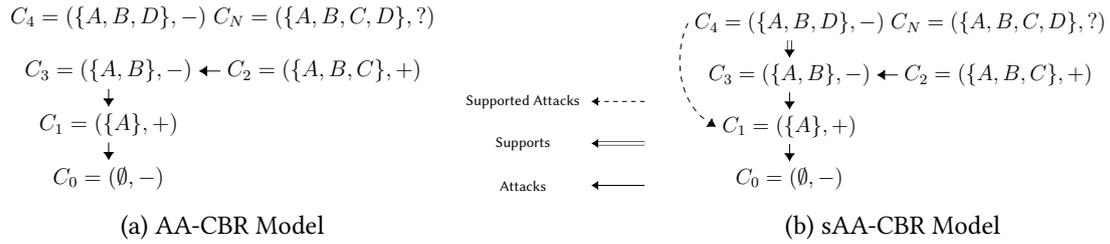
\begin{figure}
    
    \begin{subfigure}[t]{0.40\textwidth}
        \centering
        \resizebox{\examplew\textwidth}{!}{
        \begin{tikzpicture}[main/.style = {draw=none, font=\large}]
    
            \node[main] (C0) at (0, 0)  {$C_{0} = (\emptyset, - )$};
            \node[main] (C1) at (0, 1) {$C_{1} = (\{A\}, + )$};

            \node[main] (C2) at (4, 2)  {$C_{2} = (\{A, B, C\}, + )$};
            \node[main] (C3) at (0, 2)  {$C_{3} = (\{A, B\}, - )$};
            \node[main] (C4) at (0, 3)  {$C_{4} = (\{A, B, D\}, - )$};
            
            \node[main] (N) at (4, 3)  {$C_{N} = (\{A, B, C, D\}, ?)$};

            \draw[-{Latex[length=1.5mm, width=2mm]}] (C1) -- (C0);                   
            \draw[-{Latex[length=1.5mm, width=2mm]}] (C3) -- (C1);                   
            \draw[-{Latex[length=1.5mm, width=2mm]}] (C2) -- (C3);              
            
            \begin{scope}[transparency group, opacity=0.5]
            \end{scope}
    
        \end{tikzpicture} 
        }
    \caption{AA-CBR Model}
    \label{fig:example-1}
    \end{subfigure}
    \hfill
    \begin{subfigure}[t]{0.18\textwidth}
        \centering
        \resizebox{\textwidth}{!}{
        \begin{tikzpicture}[main/.style = {draw=none, font=\small}]
    
            \node[main] (C0) at (0.5, 0)  {};
            \node[main] (C1) at (2, 0) {};
            \node[main] (attacks) at (-1, 0) {Attacks};
            
            \node[main] (C2) at (0.5, 1)  {};
            \node[main] (C3) at (2, 1) {};
            \node[main] (supports) at (-1, 1) {Supports};
            
            \node[main] (C4) at (0.5, 2)  {};
            \node[main] (C5) at (2, 2) {};
            \node[main] (supp_att) at (-1, 2) {Supported Attacks};

            \draw[-{Latex[length=1.5mm, width=2mm]}] (C1) -- (C0);                   
            \draw[double, double distance = 2pt, -{Latex[length=1.5mm, width=2mm]}] (C3) -- (C2);                   
            
            \draw[dashed, -{Latex[length=1.5mm, width=2mm]}] (C5) -- (C4);
            
            \begin{scope}[transparency group, opacity=0.5]
            \end{scope}
    
        \end{tikzpicture} 
        }
    \end{subfigure}
    \hfill
    \begin{subfigure}[t]{0.40\textwidth}
        \centering
        \resizebox{\examplew\textwidth}{!}{
        \begin{tikzpicture}[main/.style = {draw=none, font=\large}]
    
            \node[main] (C0) at (0, 0)  {$C_{0} = (\emptyset, - )$};
            \node[main] (C1) at (0, 1) {$C_{1} = (\{A\}, + )$};

            \node[main] (C2) at (4, 2)  {$C_{2} = (\{A, B, C\}, + )$};
            \node[main] (C3) at (0, 2)  {$C_{3} = (\{A, B\}, - )$};
            \node[main] (C4) at (0, 3)  {$C_{4} = (\{A, B, D\}, - )$};
            
            \node[main] (N) at (4, 3)  {$C_{N} = (\{A, B, C, D\}, ?)$};
            
            \draw[-{Latex[length=1.5mm, width=2mm]}] (C1) -- (C0);                   
            \draw[-{Latex[length=1.5mm, width=2mm]}] (C3) -- (C1);                   
            \draw[-{Latex[length=1.5mm, width=2mm]}] (C2) -- (C3);              
            
            \draw[double, double distance = 2pt, -{Latex[length=1.5mm, width=2mm]}] (C4) -- (C3);                   
            \draw[dashed, -{Latex[length=1.5mm, width=2mm]}] (C4.west) .. controls (-2.2,2.5) and (-2.2,1.5) .. (C1.west);
            
            \begin{scope}[transparency group, opacity=0.5]
            \end{scope}
    
        \end{tikzpicture} 
        }
    \caption{sAA-CBR Model}
    \label{fig:example-2}
    \end{subfigure}

    \caption{The argumentation debate with and without supports.}
    \label{fig:example}
    
\end{figure}

In Figure~\ref{fig:example-1}, we see that for AA-CBR, the minimality condition means that case $C_{4}$ does not attack $C_{1}$, as $C_{3}$ is minimal to $C_{1}$. When computing which arguments are accepted\footnote{Using the grounded argumentation semantics~\cite{DUNG-aa}.}, we find the model predicts that the new case has a healthy diet $(+)$. This is because $C_{2}$ is unattacked and attacks $C_{3}$, so $C_{3}$ is rejected, and thus $C_{1}$ is accepted and $C_{0}$ is rejected. Note that $C_{4}$ did not contribute to the debate as there is no path from $C_{4}$ to $C_{0}$; $C_{4}$ is a \textit{spike}.

In contrast, Figure~\ref{fig:example-2} showcases the corresponding sAA-CBR model, in which $C_{4}$ supports $C_{3}$. We can interpret the support from $C_{4}$ to $C_{3}$ as an attack by $C_{4}$ on $C_{1}$
\footnote{This is known as a \textit{supported attack}~\cite{bipolarity-in-argumentation}. We can also add \textit{secondary attacks}. Model semantics can be found in the supplementary material.
}. Now, as $C_{1}$ is attacked, $C_{0}$ is accepted, and an unhealthy diet $(-)$ is predicted. $C_{4}$ has contributed to the debate, and the predicted outcome has changed. This new predicted outcome is desirable given that the casebase contains more evidence arguing in its favour. We can prove that for sAA-CBR, when the default case is the least exceptional, there are no spikes\footnote{The proof can be found in the supplementary material.}. 
%Furthermore, we can prove that when all \textit{nearest} cases agree on an outcome, sAA-CBR predicts this outcome\improvement{cite proof in appendix} - just like AA-CBR - preserving this key property.

One might argue that by interpreting supports in this way, we could instead remove the minimality condition altogether and not consider supports, but note that we still enforce minimality where necessary. For example in both models $C_{2}$ does not attack $C_{0}$, because $C_{1}$ exists. The total collapse of the minimality condition would lead to $C_{0}$ being automatically rejected as soon as a case with an opposing outcome, relevant to the new case, exists in the dataset, making the model ineffective for classification.

{\bf Conclusion.} We have introduced sAA-CBR, a model that uses supports to remove spikes, thus all data points contribute to debates meaningfully. Future work should explore an empirical analysis of the classification performance of sAA-CBR. Furthermore, supports can be explored in preference-based or cumulative AA-CBR~\cite{preference-based-aacbr,monotonicity-and-noise-tolerance}.

%We have introduced sAA-CBR, a model that uses supports to prevent extraneous data points, known as spikes. This ensures that all data points can contribute to classifying unlabelled cases when needed, and we prove that the model no longer contains spikes. We show how to apply semantics for bipolar argumentation frameworks to interpret supports as (supported and secondary) attacks, and that such an addition does not lead to a complete collapse of the minimality condition that AA-CBR relies on. Future work should explore an empirical analysis of sAA-CBR to show experimentally if the addition of supports and removal of spikes leads to increased classification accuracy. Furthermore, supports can be explored in other AA-CBR variants, such as preference-based AA-CBR~\tocite{aa-cbr-p} and cumulative AA-CBR~\tocite{cAA-CBR}. 

\newpage 
\appendix
\noindent {\LARGE \bf Supplementary Material }

\section{Supported Abstract Argumentation for Case-Based Reasoning Definition}
\label{appendix:saa-cbr}

A \emph{bipolar argumentation framework (BAF)}~\citep{bipolar-framework} extends the argumentation framework with the addition of a \emph{supports} relation. A BAF is represented as $\langle \args, \attacks, \supports \rangle$, where $\args$ is a set of arguments and $\attacks \subseteq \args \times \args$ is the attack relation and $\supports \subseteq \args \times \args$ is the support relation. A BAF can be represented graphically, wherein arguments are represented by nodes and edges are represented by attacks and supports.

\begin{definition}[Supported AA-CBR]
\label{def:supported-aa-cbr}
    Let $D \subseteq X \times Y$ be a finite \emph{casebase} of labelled examples where $X$ is a set of \emph{characterisations} and $Y = \{\delta, \bar{\delta}\}$ is the set of possible outcomes. Each example is of the form $(x, y)$. Let $\casedefault$ be the \emph{default argument} with $\delta$ the \emph{default outcome}. Let $N$ be an \emph{unlabelled example} of the form $\casenew$ with $y_{?}$ an unknown outcome.  
    Let $\succcurlyeq$ and $\nsim$ be a partial order (specifying exceptionalism) and binary relation (specifying irrelevance) defined over $X$, respectively. 
    The bipolar argumentation framework \baf{\xnew} mined from $D$ and $x_N$ is \mbox{$\langle\args, \attacks, \supports\rangle$} in which:

    \begin{itemize}
            \item $\args = D \cup \{\casedefault\} \cup \{N\}$
            \item for $\casealpha, \casebeta \in D \cup \{(x_{\delta}, \delta)\}$, it holds that $\casealpha \attacks \casebeta$ iff
                  \begin{enumerate}
                      \item $y_{\argalpha} \not = y_{\argbeta}$, and
                      \item One of the following holds:  
                      \begin{enumerate}
                          \item $\xalpha$ is more \emph{exceptional} than $\xbeta$ and there is \emph{minimal} difference between them:
                        \begin{enumerate}
                          \item $\x{\argalpha} \succ \xbeta$ and \label{def:aa-cbr:exceptional}
                          \item $\not\exists \casegamma \in D \cup \{(\x{\delta}, \delta)\}$ with $\ygamma = \yalpha$ and $\xalpha \succ \xgamma \succ \xbeta$; \hfill
                        \end{enumerate}
                        \item or $\xalpha$ is equivalent to $\xbeta$: 
                        \begin{enumerate}
                          \item $\xalpha = \xbeta$; \label{def:aa-cbr:symmetric-attack}
                      \end{enumerate}
                      \end{enumerate}
                  \end{enumerate}
            \item for $\casealpha \in D \cup \{(\x{\delta}, {\delta})\}$, it holds that $N \attacks \casealpha$ iff $\xnew \nsim \xalpha.$
            \item for $\casealpha, \casebeta \in D \cup \{(x_{\delta}, \delta)\}$, it holds that $\casealpha \supports \casebeta$ iff
                  \begin{enumerate}
                      \item $y_{\argalpha} = y_{\argbeta}$, and
                      \item $\xalpha$ is more \emph{exceptional} than $\xbeta$ and there is \emph{minimal} difference between them:
                      \begin{enumerate}
                      \item $\x{\argalpha} \succ \xbeta$ and
                      \item $\not\exists \casegamma \in D \cup \{(\x{\delta}, \delta)\}$ with $\xalpha \succ \xgamma \succ \xbeta$. \hfill \label{def:supported-aa-cbr:minimality}
                      \end{enumerate}
                  \end{enumerate}
    \end{itemize}
    
\end{definition}

\section{sAA-CBR Semantics}
\label{appendix:semantics}

To compute which arguments are accepted, we first translate the model to an \textit{Argumentation Framework~(AF)}~\cite{DUNG-aa}, which can be defined as a BAF with no supports and only attacks. To do this, we make use of \textit{supported attacks} and \textit{secondary attacks}~\cite{bipolarity-in-argumentation}.

\begin{definition}[adapted from~\cite{bipolarity-in-argumentation}]

Supported and secondary attacks are defined as:

\begin{itemize}
    \item There is a \textit{supported attack} from $\argalpha \in \args$ to $\argbeta \in \args$ iff there is sequence $a_{1} \supports \ldots \supports a_{n-1} \attacks a_{n}$ where $a_{1} = \argalpha$ and $a_{n} = \argbeta$. That is, there is a supported attack between $\argalpha$ and $\argbeta$ iff there is a sequence of supports from $\argalpha$ until some argument $a_{n-1}$ that attacks $\argbeta$. 
    %\todo{Should it be $a \in \args$ and $b \in \args$ instead of $\argalpha \subseteq \args$. \\Remove the comma between $a_{n-1}$ and $\attacks$ \\Add a comma between $a_1=a$ and $a_n=b$}
    
    \item There is a \textit{secondary attack} from $\argalpha \in \args$ to $\argbeta \in \args$ iff there is sequence $a_{1} \attacks \ldots \attacks a_{n-1} \supports a_{n}$ where $a_{1} = \argalpha$ and $a_{n} = \argbeta$. That is, there is a secondary attack between $\argalpha$ and $\argbeta$ iff there is a sequence of attacks from $\argalpha$ until some argument $a_{n-1}$ that supports $\argbeta$.
\end{itemize}

\end{definition}

We can then remove all supports from the BAF and construct the AF using only the attacks, supported attacks, and secondary attacks.

When used with sAA-CBR, these \textit{complex attacks} will always occur from more exceptional to less exceptional cases. The literature also defines other forms of attacks, such as mediated attacks~\cite{bipolarity-in-argumentation}, in which if an argument $\argalpha$ attacks $\argbeta$, then $\argalpha$ will also attack all arguments in the sequence of supporters of $\argbeta$. However, this clearly would lead to attacks from less exceptional to more exceptional cases, violating a key principle of AA-CBR, so we do not include these attacks.

Once the BAF has been translated to an AF, we compute the grounded semantics~\cite{DUNG-aa}, telling us which arguments to accept. We say a set of arguments $E \subseteq \args$ \emph{defends} an argument $\argbeta \in \args$ if for all $\argalpha \attacks \argbeta$ there exists $\arggamma \in E$ such that $\arggamma \attacks \argalpha$. The set of accepted arguments, called the \emph{grounded extension}, can be iteratively computed as $\mathbb{G} = \bigcup_{i \ge 0} G_i$, where $G_0$ is the set of unattacked arguments and $\forall i \ge 0$, $G_{i+1}$ is the set of all arguments that $G_i$ defends.

Finally, we have sAA-CBR$(D, \x{N}) = \delta$ if $\casedefault \in \mathbb{G}$ and $\bar \delta$ otherwise, where $\mathbb{G}$ is the grounded extension of \af{N}. 

\section{Spikes}
\label{appendix:spikes}

We can now prove that sAA-CBR contains no spikes. We begin by defining spikes formally as:

\begin{definition}
    Given a BAF $\langle \args, \attacks, \supports \rangle$ generated by sAA-CBR, then $\argalpha \in \args$ is a \textit{spike} iff there is no path in $\langle \args, \attacks, \supports \rangle$ from $\argalpha$ to $\casedefault$.
\end{definition}

We have the following theorem:

\begin{theorem}
    Given a BAF $\langle \args, \attacks, \supports \rangle$ generated by sAA-CBR where $x_{\delta}$ is the least element of $\succcurlyeq$, then the BAF contains no spikes.
\end{theorem}

\begin{proof}

Consider an argument $\casealpha \in \args$. We have that either $\casealpha$ attacks or supports $\casedefault$, which is possible as $\xalpha \succ x_{\delta}$, in which case $\casealpha$ is not a spike, or there exists some case $\casebeta$ such that $\xalpha \succ \xbeta$. Either, $\casebeta$ is minimal to $\casealpha$, in which case $\casealpha$ must attack or support $\casebeta$ or there is some other case, $\casegamma$ that is more minimal to $\casealpha$. In either case, $\casealpha$ must have at least one attack or support from another argument to one that it is most minimal to. The same reasoning applies for the arguments $\casebeta$ and $\casegamma$, which must either attack or support the default or have at least one other argument, most minimal to them, that they attack or support. This subsequent argument must either attack or support the default or another argument. We can keep applying this reasoning until the default argument is the only one to attack or support, that is, until the default argument is the most minimal to the argument we are considering. Thus, every argument has a path to the default argument.

\end{proof}

\clearpage
\begin{acknowledgments}
Research by Adam Gould was supported by UK Research and Innovation [UKRI Centre for Doctoral Training in AI for Healthcare grant number EP/S023283/1].
Francesca Toni was partially funded by the ERC under the EU’s Horizon 2020 research and innovation programme (grant agreement No. 101020934). Toni was also partially funded by J.P. Morgan and the Royal Academy of
Engineering, UK, under the Research Chairs and Senior Research Fellowships scheme.
    
\end{acknowledgments}

%%
%% Define the bibliography file to be used
\bibliography{sample-ceur}

\appendix

\end{document}